%% file: main.tex
\documentclass[11pt,letterpaper]{article}
\usepackage{emnlp2016}
\usepackage{times}
\usepackage{latexsym}
\usepackage[normalem]{ulem}
\usepackage{url} 
\usepackage{array}
\usepackage{multirow}
\usepackage{amssymb}
\usepackage{amsmath}
\usepackage{mathtools}
\usepackage{amsthm}
\usepackage{xspace}
\usepackage{algorithm}
\usepackage{tikz-dependency}
\usepackage{cancel}
\usepackage{float}
\usepackage{mathpartir}
\usepackage{proof}
\usepackage{subfigure}
\usepackage{graphicx}

\usepackage{tikz} 
\usepackage{tikz-qtree}
\usepackage{examples}
\usepackage{verbatim}

\usepackage{balance}

\usepackage{tikz-qtree-compat}
\usetikzlibrary{positioning}

\emnlpfinalcopy

\def\emnlppaperid{***}

\input{defs}

\newcommand*\trapezoid{\includegraphics{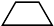}}
\newcommand*\redtrapezoid{\includegraphics{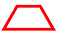}}

\title{Span-Based Constituency Parsing with a Structure-Label System and 
Provably Optimal Dynamic Oracles}

\author{James Cross \and Liang Huang\\
        School of EECS, Oregon State University, Corvallis, OR, USA\\
        {\small \tt \{james.henry.cross.iii, liang.huang.sh\}@gmail.com}
    }

\date{}

\begin{document}

\maketitle

\begin{abstract}
Parsing accuracy using efficient 
greedy transition systems 
has improved dramatically in recent years thanks 
to neural networks. 
Despite striking results in dependency parsing, however, 
neural models have not surpassed state-of-the-art approaches in constituency parsing. 
To remedy this, we introduce a new shift-reduce system 
whose stack contains merely sentence spans, represented by a bare minimum of LSTM features. 
We also design the first provably optimal dynamic oracle for constituency parsing,
which runs in amortized $O(1)$ time, 
compared to $O(n^3)$ oracles for standard dependency parsing.
Training with this oracle, we achieve the best \Fone scores on 
both English and French 
of any parser that does not use reranking or external data.
\end{abstract}

\section{Introduction}

\input{intro}

\section{Parsing System}
\label{sec:parser}

\input{parser}

\section{LSTM Span Features}
\label{sec:lstm}

\input{lstm_segment}

\section{Dynamic Oracle}
\label{sec:dyna}
\input{dyna}

\section{Related Work}
\label{sec:related}
\input{related}

\section{Experiments}
\input{experiments}

\section{Conclusion and Future Work}
\input{conclusion}

\section*{Acknowledgments}

We thank the three anonymous reviewers for comments,
Kai Zhao, Lemao Liu, Yoav Goldberg, and Slav Petrov for suggestions,
Juneki Hong for proofreading,
and Maximin Coavoux for sharing their manuscript.
This project was supported in part 
by NSF IIS-1656051, DARPA FA8750-13-2-0041 (DEFT), and 
a Google Faculty Research Award.


\balance
\bibliography{segment_parser_paper}
\bibliographystyle{emnlp2016}

\end{document}

%% file: defs.tex
\newcommand{\tuple}[1]{\ensuremath{\langle {#1} \rangle}}

\newcommand{\notes}[1]{}

 \theoremstyle{definition}
 \newtheorem{definition}{Definition}
\theoremstyle{plain}
\newtheorem{theorem}{Theorem}
\newtheorem{lemma}{Lemma}
\newtheorem{corollary}{Corollary}

\newcommand{\ith}[1]{\ensuremath{i^{{th}}}}

\newcount\permx
\newcount\permy
\def\permdot#1#2{
\permx=#1 \advance\permx by-1
\permy=#2 \advance\permy by-1
\psframe[fillcolor=black, fillstyle=solid]
(\permx,\permy)(#1, #2)
}

\newcommand{\x}[1]{\ensuremath{x_{#1}\xspace}}

\newcommand{\bracket}[3]{\ensuremath{_{\text{#2}}{\text{#1}}_{\text{#3}}\xspace}}
\newcommand{\mbracket}[3]{\ensuremath{\!\ _{{#2}}{{#1}}_{{#3}}\xspace}}

\newcommand{\boxnum}[1]{{\setlength{\fboxsep}{1pt}\raisebox{1pt}{\hspace{1pt}\fbox{\tiny #1}\hspace{1pt}}}}
\newcommand{\ind}[1]{\ensuremath{_{\kern-0.5pt\boxnum{#1}}}}

\newcommand{\oracle}{\ensuremath{\mathit{oracle}\xspace}}

\def\namecite{\newcite}

\newcommand{\Fone}{\ensuremath{\mathrm{F}_1}\xspace}

\newcommand{\smallnt}[1]{\ensuremath{_{\mbox{\tiny PP}}\xspace}}

\newcommand{\shift}{\ensuremath{\mathsf{sh}\xspace}}

\newcommand{\combine}{\ensuremath{\mathsf{comb}\xspace}}
\newcommand{\labelx}{\ensuremath{\mathsf{label}\xspace}}
\newcommand{\nolabel}{\ensuremath{\mathsf{nolabel}\xspace}}

\newcommand{\pseudocode}{Algorithm}
\floatname{algorithm}{\pseudocode}

\newcommand{\nitem}[3]{\ensuremath{\tuple{{#1},\; {#2},\; {#3}}\xspace}}

\iffalse

\else

\fi

\newcommand{\nextb}{\ensuremath{\mathit{next}\xspace}}
\newcommand{\dyna}{\ensuremath{\mathit{dyna}\xspace}}

\newcommand{\reach}{\ensuremath{\mathit{reach}\xspace}}
\newcommand{\tG}{\ensuremath{t_\mathrm{G}\xspace}}
\newcommand{\desc}{\ensuremath{\mathcal{D}\xspace}}

\newcommand{\tstar}{\ensuremath{t^*\xspace}}
\newcommand{\tstarc}{\ensuremath{\tstar(c)\xspace}}
\newcommand{\leftr}{\ensuremath{\mathit{left}\xspace}}
\newcommand{\rightr}{\ensuremath{\mathit{right}\xspace}}

\newcommand{\newmid}{\!\mid\!}

%% file: intro.tex
Parsing is an important problem in natural language processing 
which has been studied extensively for decades. 
Between the two basic paradigms of parsing,
constituency parsing, the subject of this paper,
has in general proved to be the more difficult than dependency parsing, 
both in terms of accuracy and the run time of parsing algorithms.


There has recently been a huge surge of interest in using neural networks to make parsing decisions, and such models continue to dominate the state of the art in dependency parsing \cite{andor2016globally}. In constituency parsing, however, 
neural approaches are still behind the state-of-the-art \cite{carreras2008tag,shindo2012bayesian,thang2015};
see more details in Section~\ref{sec:related}.

To remedy this, we design a new parsing framework
that is more suitable for constituency parsing, and
that can be accurately modeled by neural networks.
Observing that constituency parsing is primarily focused on sentence spans 
(rather than individual words, as is dependency parsing), 
we propose a novel adaptation of the shift-reduce system 
which reflects this focus. 
In this system, the stack consists of sentence spans rather than partial trees. It is also factored into two types of parser actions, structural and label actions, which alternate during a parse. The structural actions are a simplified analogue of shift-reduce actions, omitting the directionality of reduce actions, while the label actions directly assign nonterminal symbols to sentence spans.  


Our neural model processes the sentence once for each parse with a recurrent network. We represent parser configurations with a very small number of span features (4 for structural actions and 3 for label actions). Extending \namecite{wang2016}, each span is represented as the difference of recurrent output from multiple layers in each direction. No pretrained embeddings are required.


We also extend the idea of dynamic oracles from dependency to constituency parsing. 
The latter is significantly more difficult than the former due to \Fone being a combination of precision and recall \cite{huang:2008},
and yet we propose a simple and extremely efficient oracle (amortized $O(1)$ time). 
This oracle is proved optimal for \Fone as well as both of its components, precision and recall. Trained with this oracle, our parser achieves what we believe to be the best results for any parser without reranking which was trained only on the Penn Treebank and the French Treebank, despite the fact that it is not only linear-time, but also strictly greedy.

We make the following main contributions:

\begin{itemize}
\item A novel factored transition parsing system where the stack elements are sentence spans rather than partial trees (Section~\ref{sec:parser}).
\item A neural model where sentence spans are represented as differences of output from a multi-layer bi-directional LSTM (Section~\ref{sec:lstm}).
\item The first provably optimal dynamic oracle for constituency parsing 
which is also extremely efficient (amortized $O(1)$ time) (Section~\ref{sec:dyna}).
\item The best \Fone scores of any single-model, closed training set, parser
for English and French.
\end{itemize}

We are also publicly releasing the source code for one implementation of our parser.\footnote{\scriptsize code: \tt https://github.com/jhcross/span-parser}

%% file: parser.tex
We present a new transition-based system for constituency parsing whose fundamental unit of computation is the sentence span. It uses a stack in a similar manner to other transition systems, except that the stack contains sentence spans with no requirement that each one correspond to a partial tree structure during a parse. 

The parser alternates between two types of actions, structural and label, where the structural actions follow a path to make the stack spans correspond to sentence phrases in a bottom-up manner, while the label actions optionally create tree brackets for the top span on the stack. There are only two structural actions: {\em shift} is the same as other transition systems, while {\em combine} merges the top two sentence spans. The latter is analogous to a reduce action, but it does not immediately create a tree structure and is non-directional. Label actions do create a partial tree on top of the stack by assigning one or more non-terminals to the topmost span.

Except for the use of spans, this factored approach is similar to the odd-even parser from \namecite{mi+huang:2015}. The fact that stack elements do not have to be tree-structured, however, means that we can create productions with arbitrary arity, and no binarization is required either for training or parsing. This also allows us to remove the directionality inherent in the shift-reduce system, which is at best an imperfect fit for constituency parsing. We do follow the practice in that system of labeling unary chains of non-terminals with a single action, which means our parser uses a fixed number of steps, $(4n-2)$ for a sentence of $n$ words.

\begin{figure}[t]

\centering
\hspace{-0.15in}
  \begin{tabular}{ll}
   input:      & $w_0 \ldots w_{n-1}$ \\[0.06in]
    axiom:        & \nitem{0}{[0]}{\emptyset} \\ [0.06in]

    goal:         & \nitem{2(2n-1)}{[0,n]}{t}  \\ [0.06in]

    {\shift}       & \inferrule{\nitem{z}{\sigma\newmid j}{t}}{\nitem{z+1}{\sigma \newmid j \newmid j\!+\!1}{t}} \ $j<n$, even $z$  \\[0.2in]

    {\combine}     & \inferrule{\nitem{z}{\sigma \newmid i \newmid k \newmid j}{t}}
                             {\nitem{z+1}{\sigma \newmid i \newmid j}{t}} \ even $z$ \\[0.2in]

    {\labelx-$X$}     & \inferrule{\nitem{z}{\sigma \newmid i \newmid j}{t}}
                              {\nitem{z+1}{\sigma \newmid i \newmid j}{t\cup\{\mbracket{X}{i}{j}\}}} \ odd $z$ \\[0.2in]

    {\nolabel}     & \inferrule{\nitem{z}{\sigma \newmid i \newmid j}{t}}
                              {\nitem{z+1}{\sigma \newmid i \newmid j}{t}} \ $z\!<\!(4n\!-\!1)$, odd $z$ \\[0.2in]


  \end{tabular}
\caption{Deductive system for the Structure/Label transition parser. The stack $\sigma$ is represented as a list of integers where the span defined by each consecutive pair of elements is a sentence segment on the stack. Each $X$ is a nonterminal symbol or an ordered unary chain. The set $t$ contains labeled spans of the form \bracket{$X$}{$i$}{$j$}, which at the end of a parse, fully define a parse tree.}
\label{fig:deductive}
\end{figure}


\newlength{\mylength}
\newlength{\myheight}
\newcommand{\atr}[1]{\settowidth{\mylength}{{#1}}\resizebox{!}{!}{$\triangle$}}

\newcommand{\ntr}[1]{\settowidth{\mylength}{\,{#1}\,}\resizebox{\mylength}{0.18cm}{$\triangle$}}
\newcommand{\rtr}[1]{\settowidth{\mylength}{\,{#1}\,}\resizebox{\mylength}{0.18cm}{{\color{red}$\triangle$}}}
\newcommand{\mtr}[2]{\settowidth{\mylength}{\,{#1}\,}\settowidth{\myheight}{{#2}}\resizebox{0.8\mylength}{0.3\mylength}{$\wedge$}}
\newcommand{\mmtr}[2]{\settowidth{\mylength}{\,{#1}\,}\settowidth{\myheight}{{#2}}\resizebox{0.8\mylength}{0.2\mylength}{$\triangle$}}
\newcommand{\bmtr}[2]{{\color{blue}\mtr{#1}{#2}}}
\newcommand{\bbmtr}[2]{{\color{black}\mtr{#1}{#2}}}
\newcommand{\brmtr}[2]{{\color{black}\mmtr{#1}{#2}}}
\newcommand{\rmtr}[2]{{\color{red}\mmtr{#1}{#2}}}

\newcommand{\tr}[2]{\settowidth{\mylength}{\,{#1}\,}\settowidth{\myheight}{{#2}}\resizebox{\mylength}{1.2\myheight}{{\color{blue}$\wedge$}}}

\newcommand{\pindex}[1]{{\scriptsize\!\!\!\!\!$_{#1}$\ }}
\newcommand{\bindex}[1]{{\scriptsize\ $_{#1}$\!\!\!\!\!}}

\begin{figure*}[t]
\hspace{-0.5cm}
\begin{tabular}{cc}
\tikzset{sibling distance=0cm, level distance=0.75cm}
\raisebox{2.2cm}{
\Tree[.S [.NP [.PRP \pindex{0}I ] ] 
         [.VP [.MD \pindex{1}do ] [.VBP \pindex{2}like ] 
         [.S [.VP [.VBG \pindex{3}eating ] [.NP [.NN \pindex{4}fish\bindex{5} ] ] ] ] ] 
     ]
}
&
\resizebox{!}{2.4cm}{
\begin{tabular}{l|ll|l|p{1.7cm}}
\hline
\!\!steps & structural action & \!\!\!label action\!\! & stack after & bracket\!\!\! \\
\hline
\!\!1--2  & \shift(I/PRP) & \!\!\!\labelx-NP & $_0\scalebox{0.8}{\trapezoid}_1$ & \bracket{NP}{0}{1}\\
\!\!3--4  & \shift(do/MD)& \!\!\!\nolabel & $_0\scalebox{0.8}{\trapezoid}_1\scalebox{0.8}{\trapezoid}_2$ & \\
\!\!5--6  & \shift(like/VBP)& \!\!\!\nolabel & $_0\scalebox{0.8}{\trapezoid}_1\scalebox{0.8}{\trapezoid}_2\scalebox{0.8}{\trapezoid}_3$ & \\
\!\!7--8  & \combine & \!\!\!\nolabel & $_0\scalebox{0.8}{\trapezoid}_1\scalebox{0.8}{\trapezoid}_3$ &  \\
\!\!9--10 & \shift(eating/VBG)& \!\!\!\nolabel & $_0\scalebox{0.8}{\trapezoid}_1\scalebox{0.8}{\trapezoid}_3\scalebox{0.8}{\trapezoid}_4$ & \\
\!\!11--12 & \shift(fish/NN)& \!\!\!\labelx-NP & $_0\scalebox{0.8}{\trapezoid}_1\scalebox{0.8}{\trapezoid}_3\scalebox{0.8}{\trapezoid}_4\scalebox{0.8}{\trapezoid}_5$ & \bracket{NP}{4}{5}\\
\!\!13--14 & \combine& \!\!\!\labelx-S-VP & $_0\scalebox{0.8}{\trapezoid}_1\scalebox{0.8}{\trapezoid}_3\scalebox{0.8}{\trapezoid}_5$ & \bracket{S}{3}{5}, \bracket{VP}{3}{5}  \\
\!\!15--16  & \combine& \!\!\!\labelx-VP & $_0\scalebox{0.8}{\trapezoid}_1\scalebox{0.8}{\trapezoid}_5$ & \bracket{VP}{1}{5}\\
\!\!17--18  & \combine& \!\!\!\labelx-S & $_0\scalebox{0.8}{\trapezoid}_5$ & \bracket{S}{0}{5} \\
\hline
\end{tabular}
}\\
(a) gold parse tree &
(b) static oracle actions\\[0.2cm]
\end{tabular}
\caption{The running example. 
It contains one ternary branch and one unary chain (S-VP),
and NP-PRP-I and NP-NN-fish are {\em not} unary chains in our system.
Each stack is just a list of numbers but is visualized with spans here.
}
\label{fig:parse}
\end{figure*}

Figure~\ref{fig:deductive} shows the formal deductive system for this parser. The stack $\sigma$ is modeled as a list of strictly increasing integers whose first element is always zero. These numbers are word boundaries which define the spans on the stack. In a slight abuse of notation, however, we sometimes think of it as a list of pairs $(i,j)$, which are the actual sentence spans, i.e., every consecutive pair of indices on the stack, initially empty. We represent stack spans by trapezoids ($_i\trapezoid_j$) in the figures to emphasize that they may or not have tree stucture.

The parser alternates between structural actions and label actions according to the parity of the parser step $z$. In even steps, it takes a structural action, either combining the top two stack spans, which requires at least two spans on the stack, or introducing a new span of unit length, as long as the entire sentence is not already represented on the stack

In odd steps, the parser takes a label action. One possibility is labeling the top span on the stack, $(i,j)$ with either a nonterminal label or an ordered unary chain (since the parser has only one opportunity to label any given span). Taking no action, designated \nolabel, is also a possibility. This is essentially a null operation except that it returns the parser to an even step, and this action reflects the decision that $(i,j)$ is not a (complete) labeled phrase in the tree. In the final step, $(4n-2)$, \nolabel\  is not allowed since the parser must produce a tree.


Figure~\ref{fig:parse} shows a complete example of applying this parsing system to a very short sentence ({\em ``I do like eating fish''}) that we will use throughout this section and the next. The action in step 2 is \labelx-NP because ``I'' is a one-word noun phrase (parts of speech are taken as input to our parser, though it could easily be adapted to include POS tagging in label actions). If a single word is not a complete phrase (e.g., ``do''), then the action after a shift is \nolabel.

The ternary branch in this tree (VP $\rightarrow$ MD VBP S) is produced by our parser in a straightforward manner: after the phrase ``do like'' is combined in step 7, no label is assigned in step 8, successfully delaying the creation of a bracket until the verb phrase is fully formed on the stack. Note also that the unary production in the tree is created with a single action, \labelx-S-VP, in step 14.

The static oracle to train this parser simply consists of taking actions to generate the gold tree with a ``short-stack'' heuristic, meaning combine first whenever combine and shift are both possible.






%% file: lstm_segment.tex
Long short-term memory networks (LSTM) are a type of recurrent neural network model proposed by \namecite{hochreiter1997long} which are very effective for modeling sequences. They are able to capture and generalize from interactions among their sequential inputs even when separated by a long distance, and thus are a natural fit for analyzing natural language. LSTM models have proved to be a powerful tool for many learning tasks in natural language, such as language modeling \cite{sundermeyer2012} and translation \cite{sutskever2014sequence}.

LSTMs have also been incorporated into parsing in a variety of ways, such as directly encoding an entire sentence \cite{vinyals2015}, separately modeling the stack, buffer, and action history \cite{dyer2015transition}, to encode words based on their character forms \cite{ballesteros2015improved}, and as an element in a recursive structure to combine dependency subtrees with their left and right children \cite{kiperwasser2016easy}.

For our parsing system, however, we need a way to model arbitrary sentence spans in the context of the rest of the sentence. We do this by representing each sentence span as the elementwise difference of the vector outputs of the LSTM outputs at different time steps, which correspond to word boundaries. If the sequential output of the recurrent network for the sentence is $f_0,...,f_n$ in the forward direction and $b_n,...,b_0$ in the backward direction then the span $(i,j)$ would be represented as the concatenation of the vector differences $(f_j - f_i)$ and $(b_i - b_j)$.  

The spans are represented using output from both backward and forward LSTM components, as can be seen in Figure~\ref{fig:span_lstm_diagram}. This is essentially the LSTM-Minus feature representation described by \namecite{wang2016} extended to the bi-directional case. In initial experiments, we found that there was essentially no difference in performance between using the difference features and concatenating all endpoint vectors, but our approach is almost twice as fast.

\begin{figure}[t]
\centering
\resizebox{0.49\textwidth}{!}{\input{span_lstm_diagram}}
\caption{Word spans are modeled by differences in LSTM output. Here the span 
{\em $_3$ eating fish $_5$} is represented by the vector differences (${\mathbf{f_5}-\mathbf{f_3}}$) and (${\mathbf{b_3}-\mathbf{b_5}}$). The forward difference corresponds to LSTM-Minus \protect\cite{wang2016}.}
\label{fig:span_lstm_diagram}
\end{figure}
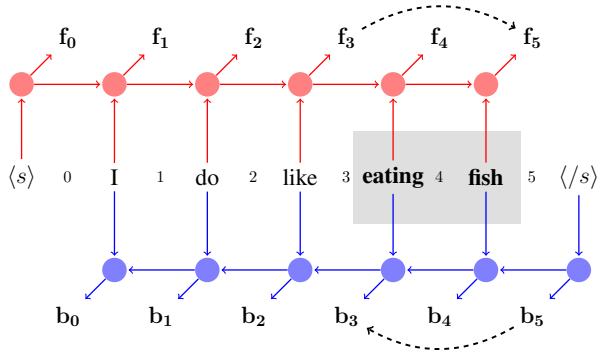



This model allows a sentence to be processed once, and then the same recurrent outputs can be used to compute span features throughout the parse. Intuitively, this allows the span differences to learn to represent the sentence spans in the context of the rest of the sentence, not in isolation (especially true for LSTM given the extra hidden recurrent connection, typically described as a ``memory cell''). In practice, we use a two-layer bi-directional LSTM, where the input to the second layer combines the forward and backward outputs from the first layer at that time step. For each direction, the components from the first and second layers are concatenated to form the vectors which go into the span features. See \namecite{cross2016acl} for more details on this approach.

For the particular case of our transition constituency parser, we use only four span features to determine a structural action, and three to determine a label action, in each case partitioning the sentence exactly. The reason for this is straightforward: when considering a structural action, the top two spans on the stack must be considered to determine whether they should be combined, while for a label action, only the top span on the stack is important, since that is the candidate for labeling. In both cases the remaining sentence prefix and suffix are also included. These features are shown in Table~\ref{table:features}.

The input to the recurrent network at each time step consists of vector embeddings for each word and its part-of-speech tag. Parts of speech are predicted beforehand and taken as input to the parser, as in much recent work in parsing. In our experiments, the embeddings are randomly initialized and learned from scratch together with all other network weights, and we would expect further performance improvement from incorporating embeddings pre-trained from a large external corpus.

The network structure after the the span features consists of a separate multilayer perceptron for each type of action (structural and label). For each action we use a single hidden layer with rectified linear (ReLU) activation. The model is trained on a per-action basis using a single correct action for each parser state, with a negative log softmax loss function, as in \namecite{chen2014fast}.

\newcommand{\hhbox}{\resizebox{1.1cm}{0.18cm}{\text{$\Box$}}}
\begin{table}[t]
\centering
\resizebox{.48\textwidth}{!}{
\begin{tabular}{l|l|l}
   \hline
   Action   &  Stack  &  LSTM Span Features  \\
   \hline
   Structural   &  $\sigma\newmid i \newmid k \newmid j$ & $_0 \hhbox _i \trapezoid _k \trapezoid _j \hhbox _n$
 \!\!\!\!\! \\
   Label        &  $\sigma\newmid i \newmid j$ & $_0 \hhbox _i \scalebox{2.25}[1]{\trapezoid} _j \hhbox _n$ \!\!\!\!\! \\
    \hline
\end{tabular}
}
\caption{Features used for the parser. No label or tree-structure features are required.}
\label{table:features}
\end{table}

%% file: span_lstm_diagram.tex
\begin{tikzpicture}[shorten >=1pt,->,draw=black!50, node distance=\layersep
    \tikzstyle{every node}=[font=\Large]
    \tikzstyle{every pin edge}=[<-,shorten <=1pt]
    \tikzstyle{neuron}=[circle,fill=black!25,minimum size=15pt,inner sep=0pt]
    \tikzstyle{input neuron}=[neuron, fill=green!50];
    \tikzstyle{output neuron}=[neuron, fill=red!50];
    \tikzstyle{hidden neuron}=[neuron, fill=blue!50];
    \tikzstyle{annot} = [font=\Large, text centered];

    \tikzstyle{forward lstm}=[neuron, fill=red!50];
    \tikzstyle{backward lstm}=[neuron, fill=blue!50];

    \fill [gray!25] (9.15, -6) rectangle (12.75, -8);

    \node[annot] (W-0) at (2, -7) {$\langle s \rangle$};
    \node[annot] (W-1) at (4, -7) {I};
    \node[annot] (W-2) at (6, -7) {do};
    \node[annot] (W-3) at (8, -7) {like};
    \node[annot] (W-4) at (10, -7) {\bf eating};
    \node[annot] (W-5) at (12, -7) {\bf fish};
    \node[annot] (W-6) at (14, -7) {$\langle /s \rangle$};

    \foreach \name / \x in {0,...,5}
        {
        \node[annot] (X-\name) at (2*\x+3, -7) {$_{\x}$};
        \node[forward lstm] (F-\name) at (2*\x+2,-5) {};
        \path[red, thick] (W-\name) edge (F-\name);
        \node[annot] (FH-\name) at (2*\x+3,-4) {$\mathbf{f_\name}$};
        \path[red, thick] (F-\name) edge (FH-\name);

        \node[annot] (BH-\name) at (2*\x+3,-10) {$\mathbf{b_\name}$};
        }

    \foreach \name / \x in {1,...,6}
        {
        \node[backward lstm] (B-\name) at (2*\x+2,-9) {};
        \path[blue, thick] (W-\name) edge (B-\name);
        }

    \foreach \left / \right in {0/1,1/2,2/3,3/4,4/5}
        {
        \path[red, thick] (F-\left) edge (F-\right);
        }

    \foreach \left / \right in {0/1,1/2,2/3,3/4,4/5, 5/6}
        {
        \path[blue, thick] (B-\right) edge (BH-\left);
        }

    \foreach \left / \right in {1/2,2/3,3/4,4/5,5/6}   
        {
        \path[blue, thick] (B-\right) edge (B-\left);
        }     

    \path (FH-3) edge[->, black, very thick, dashed, bend left] (FH-5);
    \path (BH-5) edge[->, black, very thick, dashed, bend left] (BH-3);

\end{tikzpicture}

%% file: dyna.tex
The baseline method of training our parser is what is known as a static oracle: we simply generate the sequence of actions to correctly parse each training sentence, using a short-stack heuristic (i.e., combine first whenever there is a choice of shift and combine). This method suffers from a well-documeted problem, however, namely that it only ``prepares" the model for the situation where no mistakes have been made during parsing, an inevitably incorrect assumption in practice. 
To alleviate this problem, \namecite{goldberg2013training} define a dynamic oracle 
to return the best possible action(s) at any arbitrary configuration. 

In this section, we introduce an easy-to-compute optimal dynamic oracle for our constituency parser. 
We will first define some concepts upon which the dynamic oracle is built and then show how optimal actions can be very efficiently computed using this framework. In broad strokes, in any arbitrary parser configuration $c$ there is a set of brackets \tstarc\ from the gold tree which it is still possible to reach. By following dynamic oracle actions, all of those brackets and only those brackets will be predicted.

Even though proving the optimality of our dynamic oracle (Sec.~\ref{sec:correctness}) is involved,
computing the oracle actions is extremely simple (Secs.~\ref{sec:evenodd}) and efficient 
(Sec.~\ref{sec:complexity}).

\subsection{Preliminaries and Notations}

Before describing the computation of our dynamic oracle, 
we first need to rigorously establish the desired optimality of dynamic oracle.
The structure of this framework follows \namecite{goldberg+:2014}.

\begin{definition}
We denote $c\vdash_\tau c'$ iff. $c'$ is the result of action $\tau$ on configuration $c$,
also denoted functionally as $c'=\tau(c)$.
We denote $\vdash$ to be the union of $\vdash_\tau$ for all actions $\tau$, 
and $\vdash^*$ to be the reflexive and transitive closure of $\vdash$.
\end{definition}

\begin{definition}[descendant/reachable trees]
We denote~$\desc(c)$ to be the set of final descendant trees derivable from~$c$,
i.e., $\desc(c) = \{ t \mid c \vdash^* \nitem{z}{\sigma}{t} \}$.
This set is also called ``reachable trees'' from $c$.
\end{definition}

\begin{definition}[\Fone]
We define the standard \Fone metric of a tree~$t$ with respect to gold tree $\tG$ as
\(\Fone(t)\! =\! \frac{2 r p}{r+p}, \)
where 
\(\textstyle r
= \frac{|t\cap \tG|}{|\tG|}, 
  p
= \frac{|t \cap \tG|}{|t|}.\)
\end{definition}

The following two definitions are similar to those for dependency parsing by \namecite{goldberg+:2014}.

\begin{definition}
We extend the \Fone function 
to configurations
to define the maximum possible \Fone from a given configuration:
\(\Fone(c) = \max_{t_1 \in \desc(c)} \Fone(t_1).\)
\end{definition}

\begin{definition}[oracle]
We can now define the desired dynamic oracle of a configuration $c$ to be the set of 
actions that retrain the optimal \Fone:
\[\oracle(c) = \{ \tau \mid \Fone(\tau(c)) = \Fone(c) \}. \]
\label{def:ora}
\end{definition}

This abstract oracle is implemented by $\dyna(\cdot)$ in Sec.~\ref{sec:evenodd},
which we prove to be correct in Sec.~\ref{sec:correctness}.



\begin{definition}[span encompassing]
We say span $(i,j)$ is encompassed by span $(p,q)$, notated $(i,j) \preceq (p,q)$,
iff.~$p\leq i < j \leq q$. 
\end{definition}

\begin{definition}[strict encompassing]
We say span $(i,j)$ is strictly encompassed by span $(p,q)$, notated $(i,j) \prec (p,q)$,
iff.~$(i, j) \preceq (p,q)$ and $(i,j)\neq (p,q)$.
We then extend this relation from spans to brackets,
and notate 
$\mbracket{X}{i}{j} \prec \mbracket{Y}{p}{q}$ iff.~
$(i,j) \prec (p,q)$.
\end{definition}

\begin{figure}
\centering


\begin{tikzpicture}
\begin{scope}[shift={(-0.1cm,0.35cm)}]
\tikzset{sibling distance=4.1cm, level distance=4.95cm, edge from parent/.style={draw, blue}}
\Tree[.{\small \color{blue}{\bracket{S}{0}{5}}} {} {} ] 
\end{scope}

\begin{scope}[shift={(0.25cm,-0.875cm)}]
\tikzset{sibling distance=3.3cm, level distance=3.7cm, edge from parent/.style={draw, blue, very thick}}
\Tree[.{\small \color{blue}{\bf \bracket{VP}{1}{5}} \ \  \ \ \ \ \ } {} {} ]
\end{scope}

\begin{scope}[shift={(1.1cm,-2.6cm)}]
\tikzset{sibling distance=1.6cm, level distance=2.0cm, edge from parent/.style={draw, gray, dotted, thick}}
\Tree[.{\raisebox{-0.1cm}{\small {\color{gray}{\bracket{S/VP}{3}{5}}}\ \ \ \ \ \ \ \ \ \ \ }} {} {} ]
\end{scope}

\begin{scope}[shift={(1.7cm,-3.6cm)}]
\tikzset{sibling distance=0.3cm, level distance=1.0cm, edge from parent/.style={draw, cyan}}
\Tree[.{\raisebox{-0.2cm}{\small \color{cyan}{\bracket{NP}{4}{5}}\ \ \ \ \ \ \ \ }} {} {} ]
\end{scope}

\end{tikzpicture}

\vspace{-0.7cm}

\hspace{-0.05cm}\scalebox{1.25}{$_0\scalebox{0.8}[1]{\trapezoid}_1\scalebox{0.8}[1]{\trapezoid}_2\scalebox{2.8}[1]{\redtrapezoid}_4\hspace{0.6cm}_5$}
\vspace{0.2cm}

\hspace{-0.2cm} I 
\hspace{0.3cm} do 
\hspace{0.3cm} like 
\hspace{0.1cm} eating 
\hspace{0.15cm} fish



\caption{Reachable brackets (w.r.t.~gold tree in Fig.~\ref{fig:deductive}) 
for 
$c=\nitem{10}{[0,1,2,4]}{\{\bracket{NP}{0}{1}\}}$
which mistakenly combines ``like eating''. 
Trapezoids indicate stack spans (the top one in red),
and solid triangles denote reachable brackets,
with $\leftr(c)$ in {\color{blue}blue} and $\rightr(c)$ in {\color{cyan}cyan}.
The next reachable bracket, $\nextb(c)={\color{blue}\bracket{\bf VP}{1}{5}}$, is in bold. 
Brackets \bracket{VP}{3}{5} and \bracket{S}{3}{5}
(in dotted triangle)
cross the top span (thus unreachable), and \bracket{NP}{0}{1} is already recognized
(thus not in $\reach(c)$ either).}
\label{fig:reachable}
\end{figure}

We next define a central concept, {\em reachable brackets}, which is 
made up of two parts, the left ones $\leftr(c)$ which encompass $(i,j)$ 
without crossing any stack spans,
and the right ones $\rightr(c)$ which are completely on the queue.
See Fig.~\ref{fig:reachable} for examples. 

\begin{definition}[reachable brackets]
For any configuration $c=\nitem{z}{\sigma \newmid i \newmid j}{t}$, we define the set of reachable gold brackets
(with respect to gold tree \tG) as
\[\reach(c) = \leftr(c) \cup \rightr(c) \]
where the left- and right-reachable brackets are
\begin{align}
\leftr(c) & \!=\! \{\mbracket{X}{p}{q}\in \tG \mid (i,j)\prec (p,q),\, p\in \sigma \newmid i\}\notag\\
\rightr(c) & \!=\! \{\mbracket{X}{p}{q}\in \tG \mid p \geq j\} \notag
\end{align}
for even $z$, with the $\prec$ replaced by $\preceq$ for odd $z$.

Special case (initial): $\reach(\nitem{0}{[0]}{\emptyset})=\tG$.
\end{definition}

The notation $p \in \sigma \newmid i$ simply means
$(p,q)$ does not ``cross'' any bracket on the stack.
Remember our stack is just a list of span boundaries,
so if $p$ coincides with one of them, $(p,q)$'s left boundary is not crossing
and its right boundary $q$ is not crossing either
since $q\geq j$ due to $(i,j)\prec(p,q)$.

Also note that $\reach(c)$ is strictly disjoint from $t$, i.e., $\reach(c)\cap t=\emptyset$
and $\reach(c) \subseteq \tG - t$. See Figure~\ref{fig:tstar} for an illustration.


\begin{definition}[next bracket]
For any configuration $c=\nitem{z}{\sigma \newmid i \newmid j}{t}$, 
the next reachable gold bracket (with respect to gold tree \tG) 
is the smallest reachable bracket (strictly) encompassing $(i,j)$:
\[\nextb(c) = \textstyle\min_\prec \leftr(c).\]
\end{definition}


\subsection{Structural and Label Oracles}

\label{sec:evenodd}

For an even-step configuration $c=\nitem{z}{\sigma \mid i \mid j}{t}$,
we denote the next reachable gold bracket $\nextb(c)$ to be \mbracket{X}{p}{q},
and define the dynamic oracle to be:

\begin{equation}
\dyna(c) = 
 \begin{cases}
   \{\shift\} & \text{if $p=i$ and $q>j$} \\
   \{\combine\} & \text{if $p<i$ and $q=j$} \\
   \{\shift, \combine\} & \text{if $p<i$ and $q>j$} 
 \end{cases}
\label{eq:dyna1}
\end{equation}

As a special case $\dyna(\nitem{0}{[0]}{\emptyset}) = \{\shift\}$.

Figure~\ref{fig:dyna} shows examples of this policy.
The key insight is, if you follow this policy, you will {\em not} miss the next reachable bracket,
but if you do not follow it, you certainly will. 
We formalize this fact below (with proof omitted due to space constraints) 
which will be used to prove the central results later.
\begin{lemma}
For any configuration $c$, 
for any $\tau\in\dyna(c)$, we have $\reach(\tau(c)) = \reach(c)$;
for any $\tau'\notin\dyna(c)$, we have $\reach(\tau(c)) \subsetneq \reach(c)$.
\label{lem:destroy}
\end{lemma}

\begin{figure}
\resizebox{0.48\textwidth}{!}{
\begin{tabular}{l|l|l}
\hline
\multirow{2}{*}{configuration} & \multicolumn{2}{c}{oracle}\\
                       & \!\!static\!\! & dynamic \\
\hline
$_0\scalebox{0.8}[1]{\trapezoid}_1\scalebox{0.8}[1]{\trapezoid}_2\scalebox{0.8}[1]{\redtrapezoid}_3$    & \!\!\combine\!\!\! & \{\combine, \shift\}\!\!\! \\

\small   \;\,  I  \;\,            do    \;\,   like  & &  \hspace{-0.1cm}$_1\!\!\!\bmtr{do like eating fish}{fish}\!\!\!_5$ \hspace{-1.8cm}$_2\redtrapezoid_3$  \!\!\!\\
\hline

$_0\scalebox{0.8}[1]{\trapezoid}_1\scalebox{1.85
}[1]{\redtrapezoid}_3$  & \multirow{10}{*}{\!\!\em undef.\!\!\!} & \{\shift\} \\

\small   \;\,  I  \;\,            do \;\, like  \, \raisebox{0.5cm}{\scriptsize $t\!=\!\{..., \text{\sout{\bracket{VP}{1}{3}}}\}$} \!\!\!& &  \hspace{-0.1cm}$_1\!\!\!\bmtr{do like eating fish}{fish}\!\!\!_5$ \hspace{-2.1cm}$\scalebox{1.8}[0.9]{\redtrapezoid}_3$ \!\!\! \\ [0.2cm]
\cline{1-1} \cline{3-3}

$_0\scalebox{0.8}[1]{\trapezoid}_1\scalebox{0.8}[1]{\trapezoid}_2\text{\sout{\scalebox{2.8}[1]{\redtrapezoid}}}_4$   &  & \{\combine, \shift\}\!\!\! \\

\small   \;\,  I  \;\,            do \;\, {like eating}  & & \hspace{-0.2cm} $_1\!\!\!\bmtr{do like eating fish}{fish}\!\!\!_5$ \hspace{-1.85cm}$_2\scalebox{1.6}[0.9]{\redtrapezoid}_4$ \!\!\!\\[0.2cm]
\cline{1-1} \cline{3-3}

$_0\scalebox{0.8}[1]{\trapezoid}_1\scalebox{0.8}[1]{\trapezoid}_2\text{\sout{\scalebox{3.0}[1]{\trapezoid}}}_4\scalebox{0.8}[1]{\redtrapezoid}_5$  &  & \{\combine\} \\

\small   \;\,  I  \;\,            do \;\, {like eating} \;\, fish  & & \hspace{-0.2cm} $_1\!\!\!\bmtr{do like eating fish}{fish}\!\!\!_5$ \hspace{-1.cm}$_4\scalebox{0.8}[0.9]{\redtrapezoid}$  \!\!\!\\
\hline
\end{tabular}
}
\caption{Dynamic oracle with respect to the gold parse in Fig.~\ref{fig:parse}.
The last three examples are off the gold path with strike out indicating structural or label mistakes.
Trapezoids denote stack spans (top one in red) and 
the blue triangle denotes the next reachable bracket $\nextb(c)$ 
which is \bracket{VP}{1}{5} in all cases.
}
\label{fig:dyna}
\end{figure}



The label oracles are much easier than structural ones.
For an odd-step configuration $c=\nitem{z}{\sigma \mid i \mid j}{t}$,
we simply check if $(i,j)$ is a valid span in the gold tree $t_G$
and if so, label it accordingly, otherwise no label.
More formally,

\begin{equation}
\dyna(c) = 
 \begin{cases}
   \{\text{\labelx-$X$}\} & \text{if some $\mbracket{X}{i}{j}\in t_G$} \\
   \{\nolabel\} & \text{otherwise} 
 \end{cases}
\label{eq:dyna2}
\end{equation}

\subsection{Correctness}
\label{sec:correctness}

To show the optimality of our dynamic oracle, we begin 
by defining a special tree \tstarc\/ and show that it is optimal 
among all trees reachable from configuration $c$.
We then show that following our dynamic oracle (Eqs.~\ref{eq:dyna1}--\ref{eq:dyna2}) from 
 $c$ will lead to \tstarc. 

\begin{definition}[\tstarc]
For any configuration $c=\nitem{z}{\sigma}{t}$, 
we define the optimal tree $\tstarc$ to include all reachable gold brackets
and nothing else. More formally, \(\tstarc = t\cup \reach(c).\)
\label{def:tstar}
\end{definition}

We can show by induction that \tstarc\/ is attainable:
\begin{lemma}
For any configuration $c$, the optimal tree is a descendant of $c$,
i.e., $\tstarc \in \desc(c)$.
\label{lem:tstar}
\end{lemma}


\begin{figure}
\centering
\resizebox{!}{4.5cm}{\input{venn}}
\caption{The optimal tree $\tstarc$
adds all reachable brackets and nothing else.
Note that $\reach(c)$ and $t$ are disjoint. 
}
\label{fig:tstar}
\end{figure}
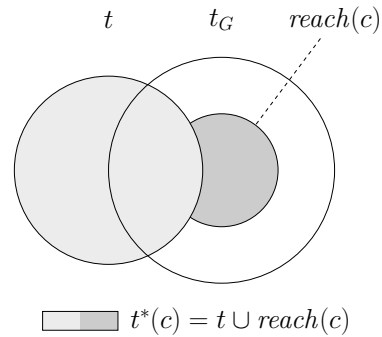

The following Theorem shows that \tstarc\/ is indeed the best possible tree:

\begin{theorem}[optimality of $t^*$]
For any configuration~$c$, 
\(\Fone(\tstarc) = \Fone(c).\)
\label{thm:tstar}
\end{theorem}

\begin{proof}
({\sc Sketch}) Since \tstarc\ adds all possible additional gold brackets (the brackets in $\reach(c)$), it is not possible to get higher recall. Since it adds no incorrect brackets, it is not possible to get higher precision. Since \Fone is the harmonic mean of precision and recall, it also leads to the best possible \Fone.
\end{proof}

\begin{corollary}
For any $c=\nitem{z}{\sigma}{t}$, for any $t'\in \desc(c)$ and $t'\neq \tstarc$, we have 
\(\Fone(t') < \Fone(c).\)
\label{cor:tprime}
\end{corollary}

We now need a final lemma about the policy $\dyna(\cdot)$ (Eqs.~\ref{eq:dyna1}--\ref{eq:dyna2})
before proving the main result.

\begin{lemma}
From any $c=\nitem{z}{\sigma}{t}$, for any action $\tau \in \dyna(c)$,
we have $\tstar(\tau(c)) = \tstarc$.
For any action $\tau' \notin \dyna(c)$,
we have $\tstar(\tau'(c)) \neq \tstarc$.
\label{lem:correct-wrong}
\end{lemma}

\begin{proof}
({\sc Sketch}) By case analysis on even/odd $z$. 
\end{proof}
We are now able to state and prove the main theoretical result of this paper
(using Lemma~\ref{lem:correct-wrong},
Theorem~\ref{thm:tstar} and Corollary~\ref{cor:tprime}):
\begin{theorem}
The function $\dyna(\cdot)$ in Eqs.~(\ref{eq:dyna1}--\ref{eq:dyna2})
satisfies the requirement of a dynamic oracle (Def.~\ref{def:ora}):
\[\dyna(c) = \oracle(c) \text{\rm \ for any configuration $c$}.\]
\end{theorem}

\subsection{Implementation and Complexity}

\label{sec:complexity}

For any configuration, our dynamic oracle can be computed in {\bf amortized constant} time
since there are only $O(n)$ gold brackets and thus bounding 
$|\reach(c)|$ and the choice of $\nextb(c)$. 
After each action, $\nextb(c)$ either remains unchanged,
or in the case of being crossed by a structural action
or mislabeled by a label action,
needs to be updated.
This update is simply tracing the parent link to the next smallest gold bracket repeatedly 
until the new bracket encompasses span $(i,j)$.
Since there are at most $O(n)$ choices of $\nextb(c)$ and there are $O(n)$ steps,
the per-step cost is amortized constant time.
Thus our dynamic oracle is much faster than the super-linear time oracle for
arc-standard dependency parsing in \namecite{goldberg+:2014}.



%% file: venn.tex
\begin{tikzpicture}

\tikzstyle{annot} = [font=\huge, text centered];

\node[annot] at (2,6) {$t$};
\node[annot] at (5,6) {$t_G$};
\node[annot] (rlabel) at (8,6) {$\mathit{reach}(c)$};

\node[draw,circle,minimum size=3cm,fill=gray!40,] (rcircle) at (5,2) {};

\draw (2,2)[fill=gray!15] circle (2.5cm);
\draw (5,2) circle (3cm);

\path[-, dashed] (rlabel) edge (rcircle);

\fill (0.25,-1.75)[fill=gray!15] rectangle (1.25,-2.25);
\fill (1.25,-1.75)[fill=gray!40] rectangle (2.25,-2.25);
\draw (0.25,-1.75)[] rectangle (2.25,-2.25);

\node[annot] at (5.5,-2) {$\tstarc = t \cup \mathit{reach}(c)$};

\end{tikzpicture}

%% file: related.tex
Neural networks have been used for constituency parsing in a number of previous instances. For example, \namecite{socher2013parsing} learn a recursive network that combines vectors representing partial trees, \namecite{vinyals2015} adapt a sequence-to-sequence model to produce parse trees, \namecite{watanabe2015transition} use a recursive model applying a shift-reduce system to constituency parsing with beam search, and \namecite{dyer2016recurrent} adapt the Stack-LSTM dependency parsing approach to this task. \namecite{durrett2015neural} combine both neural and sparse features for a CKY parsing system. Our own previous work \cite{cross2016acl} use a recurrent sentence representation in a head-driven transition system which allows for greedy parsing but does not achieve state-of-the-art results.  

The concept of ``oracles'' for constituency parsing
(as the tree that is most similar to $\tG$ among all possible trees)
was first defined and solved 
by \namecite{huang:2008} in bottom-up parsing.
In transition-based parsing,
the dynamic oracle for shift-reduce dependency parsing 
costs worst-case $O(n^3)$ time \cite{goldberg+:2014}.
On the other hand, after the submission of our paper 
we became aware of a parallel work \cite{coavoux2016acl} 
that also proposed a dynamic oracle for 
their own incremental constituency parser. 
However, it is not optimal due to dummy non-terminals 
from binarization.

%% file: experiments.tex
We present experiments on both the Penn English Treebank \cite{marcus1993building} and the French Treebank \cite{abeille2003building}. In both cases, all state-action training pairs for a given sentence are used at the same time, greatly increasing training speed since all examples for the same sentence share the same forward and backward pass through the recurrent part of the network. Updates are performed in minibatches of 10 sentences, and we shuffle the training sentences before each epoch. The results we report are trained for 10 epochs.

\begin{table}[t]
\centering
\resizebox{.5\textwidth}{!}{
\begin{tabular}{| l | c |}

    \hline
    \multicolumn{2}{| l |}{\bf Network architecture} \\
    \hline
    Word embeddings & 50  \\
    Tag embeddings & 20  \\
    Morphological embeddings$^\dagger$ & 10 \\
    LSTM layers & 2 \\
    LSTM units & 200 / direction \\
    ReLU hidden units & 200 / action type \\
    \hline \hline

    \multicolumn{2}{| l |}{\bf Training settings} \\
    \hline
    Embedding intialization & random \\
    Training epochs & 10 \\
    Minibatch size  & 10 sentences \\
    Dropout (on LSTM output) & $p=0.5$ \\
    ADADELTA parameters & $\rho=0.99$, $\epsilon=1 \times 10^{-7}$ \\
    \hline
\end{tabular}
}
\caption{Hyperparameters. $^\dagger$French only.} 
\label{tab:settings}
\end{table}

The only regularization which we employ during training is dropout \cite{Hinton2012}, which is applied with probability $0.5$ to the recurrent outputs. It is applied separately to the input to the second LSTM layer for each sentence, and to the input to the ReLU hidden layer (span features) for each state-action pair. We use the {\sc AdaDelta} method \cite{Zeiler2012} to schedule learning rates for all weights. All of these design choices are summarized in Table~\ref{tab:settings}.

In order to account for unknown words during training, we also adopt the strategy described by \namecite{kiperwasser2016}, where words in the training set are replaced with the unknown-word symbol {\tt UNK} with probability $p_{unk} = \frac{z}{z + f(w)}$ where $f(w)$ is the number of times the word appears in the training corpus. We choose the parameter $z$ so that the training and validation corpora have approximately the same proportion of unknown words. For the Penn Treebank, for example, we used $z = 0.8375$ so that both the validation set and the (rest of the) training set contain approximately 2.76\% unknown words. This approach was helpful but not critical, improving \Fone (on dev) by about 0.1 over training without any unknown words.

\subsection{Training with Dynamic Oracle}

The most straightforward use of dynamic oracles to train a neural network model, where we collect all action examples for a given sentence before updating, is ``training with exploration'' as proposed by \namecite{goldberg2013training}. This involves parsing each sentence according to the current model and using the oracle to determine correct actions for training. We saw very little improvement on the Penn treebank validation set using this method, however. Based on the parsing accuracy on the training sentences, this appears to be due to the model overfitting the training data early during training, thus negating the benefit of training on erroneous paths. 

Accordingly, we also used a method recently proposed by \namecite{ballesteros2016training}, which specifically addresses this problem. This method introduces stochasticity into the training data parses by randomly taking actions according to the softmax distribution over action scores. This introduces realistic mistakes into the training parses, which we found was also very effective in our case, leading to higher \Fone scores, though it noticeably sacrifices recall in favor of precision.

This technique can also take a parameter $\alpha$ to flatten or sharpen the raw softmax distribution. The results on the Penn treebank development set for various values of $\alpha$ are presented in Table~\ref{table:internal_oracle_results}. We were surprised that flattening the distribution seemed to be the least effective, as training accuracy (taking into account sampled actions) lagged somewhat behind validation accuracy. Ultimately, the best results were for $\alpha = 1$, which we used for final testing.

\begin{table}[H]
\centering
\begin{tabular}{| l |c c c |}
   \hline
   Model  &  LR  &  LP &  \Fone  \\
   \hline
   Static Oracle             & 91.34 & 91.43 & 91.38  \\
   \hline
   Dynamic Oracle            & 91.14 & 91.61 & 91.38  \\
   \hline
   + Explore ($\alpha\!=\!0.5$)  & 90.59 & 92.18 & 91.38   \\
   + Explore ($\alpha\!=\!1.0$)  & 91.07 & 92.22 & {\bf 91.64}   \\
   + Explore ($\alpha\!=\!1.5$) &  91.07 & 92.12 & 91.59  \\
    \hline
\end{tabular}
\caption{Comparison of performance on PTB development set using different oracle training approaches.}
\label{table:internal_oracle_results}
\end{table}

\subsection{Penn Treebank}

Following the literature, we used the Wall Street Journal portion of the Penn Treebank, with standard splits for training (secs 2--21), development (sec 22), and test sets (sec 23). Because our parsing system seamlessly handles non-binary productions, minimal data preprocessing was required. For the part-of-speech tags which are a required input to our parser, we used the Stanford tagger with 10-way jackknifing. 



\begin{table}[t]
\centering
\resizebox{.5\textwidth}{!}{
\begin{tabular}{| l | c c c |}
   \hline
   Closed Training \& Single Model  &  LR  &  LP &  \Fone  \\
   \hline
   \namecite{sagae2006}       &  88.1  &  87.8  &  87.9  \\
   \namecite{petrov+klein:2007}& 90.1  &  90.3  &  90.2  \\
   \namecite{carreras2008tag}  &  90.7  &  91.4  &  91.1  \\
   \namecite{shindo2012bayesian}  &     &        &  91.1  \\
   $\dagger$\namecite{socher2013parsing}   &     &        &  90.4  \\
   \namecite{zhu2013}        &  90.2  &  90.7  &  90.4  \\
   \namecite{mi+huang:2015}  &  90.7  &  90.9  &  90.8  \\
   $\dagger$\namecite{watanabe2015transition}  &  &  &    90.7  \\

   \namecite{thang2015} (A*)  &  90.9  &  91.2  &  91.1  \\
   $\dagger$*\namecite{dyer2016recurrent} (discrim.)       &        &        &  89.8  \\
   $\dagger$*\namecite{cross2016acl} & & & 90.0\\
   \hline
   $\dagger$*{\bf static oracle}          &  90.7  &  91.4  &  91.0  \\
   $\dagger$*{\bf dynamic/exploration}    &  90.5  &  92.1  &  {\bf 91.3}  \\
    \hline \hline 
   External/Reranking/Combo & & & \\
   \hline
   $\dagger$\namecite{henderson:2004} (rerank)        &  89.8  &  90.4  &  90.1  \\
   \namecite{mcclosky2006}           &  92.2  &  92.6  &  92.4  \\
   \namecite{zhu2013} (semi)        &  91.1  &  91.5  &  91.3  \\
   \namecite{huang:2008} (forest)    &        &        & 91.7 \\
   $\dagger$\namecite{vinyals2015} ({\sc wsj})$^\ddagger$\!\!\!     &        &        &  90.5  \\
   $\dagger$\namecite{vinyals2015}   (semi)   &        &        &  92.8  \\
   $\dagger$\namecite{durrett2015neural}$^\ddagger$  &  &  &  91.1 \\
   $\dagger$\namecite{dyer2016recurrent} (gen.~rerank.)\!\!\!\!        &        &        &  92.4  \\

   \hline
\end{tabular}
}
\caption{Comparison of performance of different parsers on PTB test set. $\dagger$Neural. *Greedy. $^\ddagger$External embeddings.}
\label{table:external_comparison_results}
\end{table}

Table~\ref{table:external_comparison_results} compares test our results on PTB to a range of other leading constituency parsers. Despite being a greedy parser, when trained using dynamic oracles with exploration, it achieves the best \Fone score of any closed-set single-model parser.

\subsection{French Treebank}

We also report results on the French treebank, with one small change to network structure. Specifically, we also included morphological features for each word as input to the recurrent network, using a small embedding for each such feature, to demonstrate that our parsing model is able to exploit such additional features. 

We used the predicted morphological features, part-of-speech tags, and lemmas (used in place of word surface forms) supplied with the SPMRL 2014 data set \cite{seddah2014introducing}. It is thus possible that results could be improved further using an integrated or more accurate predictor for those features. Our parsing and evaluation also includes predicting POS tags for multi-word expressions as is the standard practice for the French treebank, though our results are similar whether or not this aspect is included.

\begin{table}[t]
\centering
\resizebox{.5\textwidth}{!}{
\begin{tabular}{| l | c c c |}
   \hline

   Parser  &  LR  &  LP &  \Fone  \\
   \hline
   \namecite{bjorkelund2014introducing}$^{*,\ddagger}$  &  &  & 82.53  \\
   \namecite{durrett2015neural}$^\ddagger$          &  &  &  81.25 \\
   \namecite{coavoux2016acl}             &  &  &  80.56 \\
   \hline
   {\bf static oracle}          &  83.50 & 82.87 & 83.18  \\
   {\bf dynamic/exploration}    &  81.90 & 84.77 & {\bf 83.31}  \\
    \hline

   \hline
\end{tabular}
}
\caption{Results on  French Treebank.
$^*$reranking, $^\ddagger$external.}
\label{table:external_french_results}
\end{table}

We compare our parser with other recent work in Table~\ref{table:external_french_results}. We achieve state-of-the-art results even in comparison to \namecite{bjorkelund2014introducing}, which utilized both external data and reranking in achieving the best results in the SPMRL 2014 shared task.

\subsection{Notes on Experiments}

For these experiments, we performed very little hyperparameter tuning, due to time and resource contraints. We have every reason to believe that performance could be improved still further with such techniques as random restarts, larger hidden layers, external embeddings, and hyperparameter grid search, as demonstrated 
by \namecite{weiss2015google}.

We also note that while our parser is very accurate even with greedy decoding, the model is easily adaptable for beam search, particularly since the parsing system already uses a fixed number of actions. Beam search could also be made considerably more efficient by caching post-hidden-layer feature components for sentence spans, essentially using the precomputation trick described by \namecite{chen2014fast}, but on a per-sentence basis.

%% file: conclusion.tex
We have developed a new transition-based constituency parser which is built around sentence spans. It uses a factored system alternating between structural and label actions. We also describe a fast dynamic oracle for this parser which can determine the optimal set of actions with respect to a gold training tree in an arbitrary state. Using an LSTM model and only a few sentence spans as features, we achieve state-of-the-art accuracy on the Penn Treebank for all parsers without reranking, despite using strictly greedy inference.

In the future, we hope to achieve still better results using beam search, which is relatively straightforward given that the parsing system already uses a fixed number of actions. 
Dynamic programming \cite{huang+sagae:2010}
could be especially powerful in this context given the very simple feature representation used by our parser,
as noted also by \namecite{kiperwasser2016}. 